\documentclass[twocolumn]{IEEEtran}
\ifCLASSINFOpdf
\else
\fi
\usepackage{amsmath}
\usepackage{color}

\usepackage{graphicx}
\usepackage{caption}
\usepackage{subcaption}

\usepackage{amsmath,amsfonts,amsthm,bm}

\newtheorem{theorem}{Theorem}[section]
\newtheorem{proposition}[theorem]{Proposition}

\newtheorem{definition}[theorem]{Definition}
\newtheorem{corollary}[theorem]{Corollary}
\newtheorem{lemma}[theorem]{Lemma}

\newcommand{\ignore}[1]{}

\def\secref#1{Section~\ref{#1}}

\def\eqref#1{Eq.~\ref{#1}}

\def\1{\bm{1}}

\newcommand{\reals}{\mathbb{R}}

\def\vzero{{\bm{0}}}

\def\vc{{\bm{c}}}

\def\ve{{\bm{e}}}

\def\vu{{\bm{u}}}

\def\vw{{\bm{w}}}
\def\vx{{\bm{x}}}
\def\vX{{\bm{X}}}

\def\vy{{\bm{y}}}
\def\vz{{\bm{z}}}

\def\mA{{\bm{A}}}

\def\mI{{\bm{I}}}

\def\mX{{\bm{X}}}

\DeclareMathAlphabet{\mathsfit}{\encodingdefault}{\sfdefault}{m}{sl}
\SetMathAlphabet{\mathsfit}{bold}{\encodingdefault}{\sfdefault}{bx}{n}

\newcommand{\ourapp}{MAXMIN-SDR}
\newcommand{\comment}[1]{}

\newcommand{\diag}[1]{{\rm{diag}}\left({#1}\right)}

\DeclareMathOperator{\sign}{sign}

\begin{document}

%
\title{Maximin Optimization for Binary Regression}

%
%
%

\author{Nisan~Chiprut,
        Amir~Globerson,
        and~Ami~Wiesel
\thanks{ The authors are with the Hebrew University of Jerusalem, Tel Aviv University, and
Google Research. This research was partially supported by the Israel Science
Foundation (ISF) under Grant 1339/15}}

%
%

\markboth{Journal of \LaTeX\ Class Files, October~2020}%
{Shell \MakeLowercase{\textit{et al.}}: Bare Demo of IEEEtran.cls for IEEE Communications Society Journals}
%



\maketitle


\begin{abstract}

We consider regression problems with binary weights. Such optimization problems are ubiquitous in quantized learning models and digital communication systems. A natural approach is to optimize the corresponding Lagrangian using variants of the gradient ascent-descent method. Such maximin techniques are still poorly understood even in the concave-convex case. The non-convex binary constraints may lead to spurious local minima. Interestingly, we prove that this approach is optimal in linear regression with low noise conditions as well as robust regression with a small number of outliers. Practically, the method also performs well in regression with cross entropy loss, as well as  non-convex multi-layer neural networks. Taken together our approach highlights the potential of saddle-point optimization for learning constrained models.

\end{abstract}


%
\IEEEpeerreviewmaketitle

\section{Introduction}
Optimizations with binary constraints are ubiquitous in signal processing and machine learning.  Two important examples are digital communication systems \cite{verdu1998multiuser} and quantized statistical models \cite{binaryConnect}. Even with a convex objective, the binary set is combinatorial and leads to computationally hard minimizations. A continuous approach is to express the feasible set using quadratic constraints, introduce dual variables and obtain a saddle point optimization of the associated Lagrangian \cite{boyd2004convex}. Indeed, this is the starting point of the seminal semidefinite relaxation (SDR) approach to binary least squares \cite{QCQP_SDR}. Unfortunately, SDR scales poorly and is mostly applicable to the quadratic case. Independently, the last years witnessed a growing body of works on saddle point optimization methods due to the success of Generative Adversarial Networks (GAN). Therefore, we consider the use of these maximin techniques for convex optimizations with binary constraints, with an emphasis on binary regressions. Specifically, we prove that in the realizable setting and under favourable conditions on the data and loss function, such procedures will in fact converge to the true binary variables. 

Linear regression with binary weights, also known as binary least squares, has a long history. A classical application is detection of binary symbols in multiple input multiple output (MIMO), multiuser communication systems \cite{verdu1998multiuser} and recently binary compressed sensing \cite{fosson2018non}. Naive algorithms simply take the sign of the unconstrained solution. Modern methods rely on successive cancellation machinery \cite{verdu1998multiuser}, quadratic regularization \cite{fosson2019recovery}, and efficient search techniques as sphere decoding \cite{jalden2005complexity}. A state of the art approach is SDR that approximates the problem using a convex semidefinite program \cite{ma2010semidefinite}. SDR is the exact solution to the Lagrange dual program. In large systems, SDR gives a constant factor approximation to the optimal solution \cite{BNN_SDR}, and  achieves full diversity order \cite{jalden2008diversity}. Unfortunately, it scales poorly and it is not clear how to extend it beyond the quadratic Gaussian setting. 

More recently, the topic of learning quantized models has attracted considerable attention due to the deployment of complex deep networks on mobile devices with limited memory and energy resources. The most common approach to building compact models is to use lower bit-resolution for the weights and potentially also for the activation function. As an extreme case, researchers consider the case of binary weights. Clearly, this considerably complicates the learning process, since the discrete nature of the weights seems to make gradient descent inapplicable. The key approach for handling this challenge is straight-through estimators (STE) \cite{binaryConnect}, which ignore the non-differentiable sign function in the backward pass. STE can also be extended to more general low-precision weights, non-differentiable quantization operations and activations \cite{Hubara:2017:QNN,XNOR,dorefaNet}. The Gumble softmax trick was also recently used together with STE \cite{BNN_gumbelSoftmax} . See \cite{bethge2019back} for a summary of common STE practices in network quantization. Altogether, while 
these are sometimes empirically effective, this is not always the case (e.g. \cite{yin2019understanding}), and there is little theory to explain when they are expected to work.   

The motivation to our work are the recent advances in saddle-point optimization \cite{razaviyayn2020nonconvex}. Historically, these are known to be difficult, unstable and poorly understood. Yet, the overwhelming success in training GANs \cite{goodfellow2014generative}, and the growing research on adversarial machine learning, led to a renewed interest in this topic. The simplest method is gradient descent-ascent (GDA) which performs  alternating update steps on the two variables.
Its advantages are simplicity and scalability.
Unfortunately, even in simple convex-concave problems it is known that GDA may diverge. 
Therefore, there is an increased interest in design and analysis of improved variants, e.g. GDA$_\infty$, Extra Gradient (EG) and the Optimistic GDA (OGDA)
\cite{daskalakis2018limit,on_gda,localMinimax}. In non-convex binary minimizations, the problem is even more challenging as there may be no saddle point or spurious points. Together, these all give rise to a natural question: Can maximin methods solve the Lagrangian relaxation in binary optimizations? In what follows, we show that the answer is positive under favourable data conditions.




The goal of this paper is to introduce and analyze saddle point optimization of the Lagrangians in binary regression. From a theoretical perspective, we prove global optimality under two settings:
\begin{itemize}
    \item Low-noise binary linear regression.
    \item Robust linear regression with few outliers.
\end{itemize}
The proofs are based on the general saddle point conditions in \cite{localMinimax}, and consist of showing the {\em{existence}} and {\em{uniqueness}} of a maximin point identical to the true parameters. In this sense, our contribution is analogous to showing non spurious local minima in gradient descent \cite{ge2016matrix}, which then under certain conditions implies convergence to global optimum \cite{LeePPSJR19}. Here we demonstrate this can also be done for saddle-point optimization using GDA$_\infty$. From a practical perspective, we show using numerical experiments that GDA succeeds to solve binary regressions with cross entropy loss in real datasets. Even in non-convex shallow networks, GDA shows promising performance.

The paper is organized as follows. We begin in Section \ref{sec:main} by
formulating the general optimization problem, introducing our approach and presenting the main theorem.
In Section \ref{sec:mse}, we describe Binary linear regression and the assumptions which guarantee perfect reconstruction. We follow the same logic in Section \ref{sec:huberloss} focusing on Robust Linear regression using Huber's loss. 
In section \ref{sec:numerical_exp}, we present results of numerical experiments on both simulated and real world datasets. 

The following notations are used in the paper:
We use diag$(\cdot)$ to denote a diagonal matrix with its argument as the diagonal, and  $\boldsymbol{1}$ to denote a column vector of ones.
The multivariate Normal distribution with mean $\bm{\mu}$ and covariance $\bm{\Sigma}$ is defined as $\mathcal{N}\left(\bm{\mu}, \bm{\Sigma} \right)$. and Laplace distribution with location $a$ and scale $\sigma$ is defined as ${\mathcal{L}aplace}(a,\sigma)$.
$\vX\succeq {\bm{0}}$ means that $\vX$ is a symmetric positive semidefinite matrix and $\lambda_{\min} \left( \mX \right)$ is its smallest eigenvalue.




\section{Convex optimization with binary constraints}\label{sec:main}
We consider the minimization of convex functions subject to binary constraints
\begin{eqnarray}\label{BG}
p^* = \left\{\begin{array}{ll}
   \min_\vw  & f \left( \vw\right)\\ {\rm s.t.} 
     & w^2_i=1,\forall i
\end{array}\right.
\end{eqnarray}
where $f$ is a convex function. 

This problem is non-convex because of the constraints, and is NP-hard to solve even in the convex quadratic case \cite{verdu1998multiuser}. In what follows, we derive the Lagrangian formulation to the problem, and show that it can be solved to optimality under some conditions.

The Lagrangian associated with the above optimization problem is:
\begin{eqnarray}\label{BG_lagrange}
\mathcal{L} \left( \vw,\vz \right)  & = f \left( \vw \right) + \vw^{T} \diag{\vz} \vw - \vz^T{\bm 1}
\end{eqnarray}
A basic fact of Lagrangian optimization is that:
\begin{eqnarray}
p^* = \min_\vw\max_\vz \;\mathcal{L}(\vw,\vz)
\end{eqnarray}
The dual optimization problem is defined as the maximin problem:
\begin{eqnarray}\label{maxmin}
d^* = \max_\vz\min_\vw \;\mathcal{L}(\vw,\vz)
\end{eqnarray}
Weak duality guarantees that $d^*\leq p^*$. There are two major downsides to Lagrange duality in settings with non-convex constraints like ours:
\begin{itemize}
    \item The duality gap is typically greater than zero, i.e. $d^*<p^*$.
    \item The inner minimization may be intractable, unbounded and/or may have spurious local minima.
\end{itemize}
Our main contribution is to show that, under certain conditions, both the above difficulties disappear: there is zero duality gap, namely $d^*=p^*$, and there exists a unique maximin point.

In order to solve (\ref{maxmin}) a natural approach is to use gradient descent-ascent. Namely, perform a gradient descent step on $\vw$ followed by an ascent step on $\vz$. There are many variants of this approach. For example, taking more steps in the direction of $\vz$ \cite{on_gda}, or using update corrections as in OGD \cite{daskalakis2018limit}. 
In the non-convex case, these approaches may diverge, cycle or converge to sub-optimal points.
Recently, \cite{JinMinimax2019} introduced the notion of local maximin points, which is meant to capture the limit points of GDA. In particular, \cite{JinMinimax2019} considered a gradient flow version of GDA, where the the inner minimization step size is $\eta$ and outer maximization step is $\eta/\gamma$. They then considered the limit $\gamma\to\infty$ and called the resulting method GDA$_{\infty}$. They were able to show that, up to some degeneracy points, the limit points of GDA$_{\infty}$ were precisely the local maximin points defined next.


\begin{definition}[Local maximin \cite{JinMinimax2019}]
A point $\left(\vw^{\ast},\vz^{\ast}\right)$ is said to be a local maximin point of $\mathcal{L}$, if there exists $\delta_{0}>0$ and a function h satisfying $h(\delta)\rightarrow0$ as $\delta\rightarrow0$, such that for any $\delta\in(0,\delta_{0})$, and any $(\vw,\vz)$ satisfying $\left|\vw-\vw^{\ast}\right|\leq\delta$ and $\left|\vz-\vz^{\ast}\right|\leq\delta$, 
we have
\begin{equation}
\min_{\vw':\left|\vw-\vw^{\ast}\right|\leq h\left(\delta\right)}\mathcal{L} \left(\vz,\vw'\right)\leq \mathcal{L} \left(\vz^{\ast},\vw^{\ast}\right)\leq \mathcal{L} \left(\vz^{\ast},\vw\right).
\end{equation}
\end{definition}

In what follows, we give conditions that ensure the existence of a unique local maximin point. Furthermore, we show that in this case, it holds that $d^*=p^*$.


We begin by defining a property of a function that will play a key role in our optimality conditions.
\begin{definition}
[Sub-quadratic function]
A function $f$ is called sub-quadratic if $\forall \vw_1, \vw_2$ s.t. $f\left(\vw_1\right) < f\left(\vw_2\right)$
\begin{multline}
\left(\vw_1-\vw_2\right)^{T} \nabla f\left(\vw_2\right) \nonumber \\ 
 \quad +\frac{1}{2}\left(\vw_1-\vw_2\right)^{T}\nabla^{2}f\left(\vw_2\right)\left(\vw_1-\vw_2\right) < 0
\end{multline}
\end{definition}
In words, a function is sub-quadratic if, roughly, it is always greater than its second order Taylor approximation.
Interestingly, for sub-quadratic $f$ there exists a simple condition that guarantees a unique maximin point with zero duality gap.
To this end, we will state the main theorems of this article and briefly explain their implications. The full proofs can be found in the Appendix.
\begin{theorem} \label{general_convex_thm}
Consider the saddle point optimization in (\ref{maxmin}), assume $f$ is a sub-quadratic function, and let $\vw^{\ast} \in \{\pm1\}^n$ satisfy:
\begin{eqnarray}\label{gH_cond}
&&\|\nabla f \left( \vw^\ast \right) \|_\infty < \lambda_{\min} \left(   \nabla^2 f \left( \vw^\ast \right) \right)
\end{eqnarray}
Then, $p^*=d^*$ and $\left( \vw^{\ast}, \frac{1}{2} \diag{ \vw^{\ast}} \nabla f \left( \vw^\ast \right) \right)$ is the unique local maximin point.
\end{theorem}

\begin{corollary}
Under the conditions of Theorem \ref{general_convex_thm} the stable limit points of GDA$_{\infty}$ on the Lagrangian in (\ref{maxmin}) are $\vw^*$, up to some degenerate points.
\end{corollary}
The corollary follows from Theorem 28 in \cite{localMinimax}, which shows that limit points of GDA$_{\infty}$ are the local minimax points, and Theorem \ref{general_convex_thm} guarantees that there {\em{exists}} a {\em{unique}} local minimax point.




\section{Binary Linear Regression}\label{sec:mse}
In this section, we consider the special case of linear regression where $f$ is a quadratic loss function. Let $\vX$ denote the data matrix and $\vy$ the label values. The following theorem states that if $\vy$ were generated by a binary $\vw^*$ plus noise, and the noise is not too large, then Theorem \ref{general_convex_thm} can be applied, resulting in a single maximin point and zero duality gap.
\begin{theorem}[Binary Linear Regression]
\label{blr_convex_thm}
Let $\vX\in\reals^{m,n}$ be a data matrix, and consider the function:
\begin{equation}
 f \left(\vw\right) = \|\vX\vw-\vy\|^2_2
\end{equation}
Assume $\vy$ is generated via a binary vector $\vw^*\in\{\pm 1\}^n$ plus noise $\ve$: $\vy=\vX\vw^*+\ve$.
Assume the noise satisfies:
\begin{equation}
\|\vX^T\ve\|_\infty < \lambda_{\min} \left( \vX^T\vX \right)
\label{eq:l2_cond}
\end{equation}
Then, $p^*=d^*$ and $\left( \vw^{\ast},-\diag{ \vw^{\ast}} \mX^{T}\ve \right)$ is the unique local maximin point. 
\end{theorem}


A natural question given the above theorem is whether this condition is satisfied in practical settings. As we claim below, this in fact happens with high probability under very natural conditions on the generation process of $\vX$. 
\begin{theorem}
Assume $m>n$ and the data matrix \\$\vX\in {\mathbb{R}}^{m,n}$ is generated with standard Gaussian elements. Let the noise $\ve\in {\mathbb{R}}^{m}$ satisfy $e_i\sim{\mathcal{N}}(0,\sigma^2)$.
Then with probability exceeding 
$ 1 - 13n\sqrt{\frac{2}{\pi}} e^{-\frac{1}{8}m}$, we have
$p^*=d^*$ and $\left( \vw^{\ast},-\diag{ \vw^{\ast}} \mX^{T}\ve \right)$ is the unique local maximin point.

\end{theorem}
\begin{proof}
The theorem is a direct result of (\ref{eq:l2_cond}) and Theorem 3.3 in \cite{BNN_SDR}.
\end{proof}

It is interesting to compare these theorems to existing results on quadratic optimization with binary constraints. In particular, it is well known that the dual maximin can be solved using semidefinite relaxation (SDR) \cite{goemans1995improved}. Under a Gaussian setting, the condition in \cite{BNN_SDR} for perfect recovery using SDR is identical to (\ref{eq:l2_cond}).
Theorem \ref{blr_convex_thm} complements it by showing that there exists a unique local maximin in this setting.

\section{Binary Robust Linear Regression}\label{sec:huberloss}
In this section, we consider a more challenging robust regression loss. It is well known that the squared loss suffers from sensitivity to outliers, a common approach to address this is to use robust losses, such as the Huber loss \cite{huber2004robust}, which results in non-quadratic optimization. Here we show that in this case we can also employ Theorem \ref{general_convex_thm}.

Specifically, our results will show that, when the noise is negligible except for a few large outliers, binary robust regression has a unique saddle point with zero duality gap.

First, recall the Huber loss:
\begin{equation}
\ell \left( z \right) = 
     \begin{cases}
       \frac{z^2}{2} &\quad |z|\leq \delta\\ 
       |z|\delta - \frac{\delta^2}{2} &\quad \text{else} 
     \end{cases}  
\end{equation}

\begin{theorem}[Binary Robust Regression]\label{robust_convex_thm}
Let $\vX\in\reals^{m,n}$ be a data matrix, and consider the function:
\begin{equation}
f \left( w\right) = \sum_i\ell \left( \vx_i^T\vw-y_i \right)
\end{equation}
Assume $\vy$ is generated via a binary weight vector $\vw^*\in\{\pm 1\}^n$ plus noise: $\vy=\vX\vw^*+\ve$, and assume the noise satisfies:
\begin{eqnarray} \label{robust_cond}
&&\label{huber_condition}\|\vX^T\vc\|_\infty < \lambda_{\min} \left( \sum_i d_i \vx_i\vx_i^T \right) 
\end{eqnarray}
where 
\begin{eqnarray}
c_i = 
     \begin{cases}
       e_i &\quad |e_i| \leq \delta \\
       {\rm{sign}}(e_i) &\quad \text{else} 
     \end{cases} \qquad
d_i = 
     \begin{cases}
       1 &\quad |e_i| \leq \delta \\
       0 &\quad \text{else} 
     \end{cases}.
\end{eqnarray}
And in addition, for every $i$: 
\begin{eqnarray}\label{robust_cond_2}
|e_i| \leq \delta  \quad \text{or} \quad |e_i| \geq \delta+2\| x_i\|
\end{eqnarray}
Then we have $p^*=d^*$ and $\left( \vw^{\ast},-\diag{ \vw^{\ast}} \mX^{T}\vc \right)$ is the unique local maximin point.
\end{theorem}


Next, as we did in the quadratic case, we show that under natural conditions on the generation process of $\mX$, the above condition holds with high probability, and therefore a unique local maximin point is guaranteed.
\begin{theorem} \label{huber_diversity}
Assume $m>n$, i.e. more examples than features, and the data matrix $\vX\in {\mathbb{R}}^{m,n}$ is generated with standard Gaussian elements. Let the noise $\ve\in {\mathbb{R}}^{m}$ be a sparse vector such that on $pm$ indices $e_i$ is arbitrary large and $e_i=0$ on the rest, for a constant $0\leq p<\frac{1}{2}$. 
Then with probability exceeding $1 - 5n \sqrt{\frac{m}{\pi}}e^{- \frac{1}{16}m}$
we have $p^*=d^*$ and $\left( \vw^{\ast},-\diag{ \vw^{\ast}} \mX^{T}\vc \right)$ is the unique local maximin point.
\end{theorem}

\section{Empirical Evaluation \label{sec:numerical_exp}}
In this section, we demonstrate the performance advantages of binary maximin optimization via numerical experiments. In all the simulations the maximin problems were solved using TensorFlow \cite{abadi2016tensorflow}. We experimented with different optimizers including GDA, OGD and extra-gradient, but all gave quite similar results with no clear winner. We report the result using 
an Adam optimizer \cite{adam} combined with exponential decay of $\gamma > 1$. 
 We refer to our method as MAXIMIN.

We experiment with the following regressors:
\begin{itemize}
    \item Linear Regression (LR): The simplest approach which solves the unconstrained regression problem and then rounds the weights to $\pm 1$.
    \item Linear Programming Relaxation (LPR): 
    A more advanced convex relaxation that relaxes the constraints to\\ $w_i \in [-1,1] , \forall i$.
    \item Straight-Through Estimator (STE) \cite{binaryConnect}: Optimize the regression error using the STE estimator. Namely, set the model weights to be $\mbox{sign}(w)$ in the forward pass, and just $w\cdot \boldsymbol1_{w \in [-1,1]}$
    in the backward pass.
    \item Semi-Definite Relaxation (SDR) \cite{goemans1995improved}: The seminal SDR of Geomans and Williamson implemented via an interior point method. 
    \item Binary Optimizer (BOP) \cite{BOP}, a recently introduced optimizer, specified to train deep neural network with binary weights, it does so by sampling the model weights from Bernoulli distribution according to moving average of past gradients, this method has shown to outperform STE on various scenarios.
\end{itemize}

\subsection{Synthetic Binary Regression}

\begin{figure}
    \includegraphics[width=1.0\linewidth]{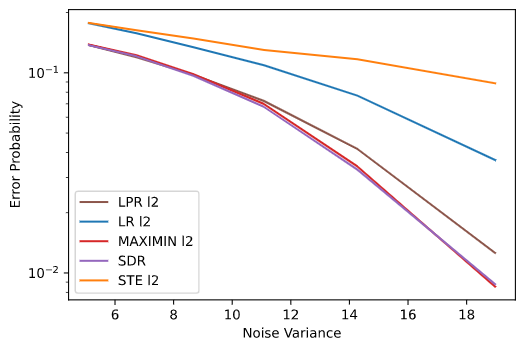}
    \caption{Linear regression with $p=30,n=60$}
    \label{small_scale_blr}
\end{figure}

\begin{figure}
    \includegraphics[width=1.0\linewidth]{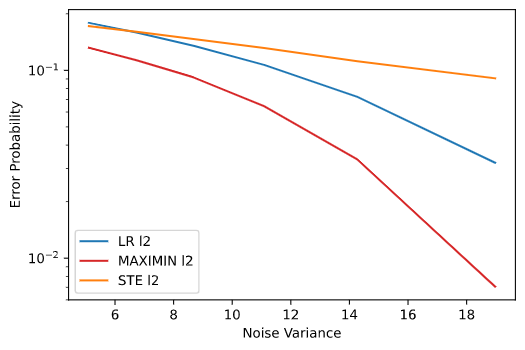}
    \caption{Linear regression with $p=2000,n=4000$}
    \label{large_scale_blr}
\end{figure}

\begin{figure}
    \includegraphics[width=1.0\linewidth]{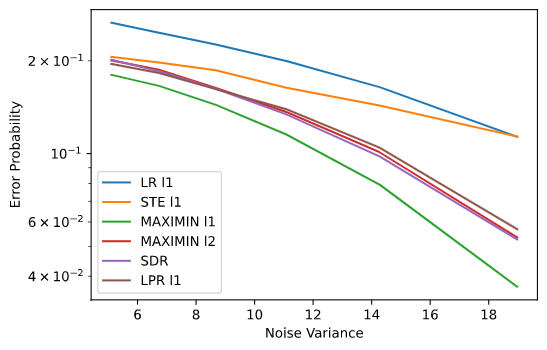}
    \caption{Robust Linear regression with $p=30,n=60$}
    \label{large_scale_rblr}
\end{figure}


We begin with synthetic experiments in linear regression with L$_2$, where data is generated as follows:
\begin{eqnarray}
\vy=\vX\vw^*+\ve \ \ , \ \ \vw^*\in \{\pm 1\}^n , \nonumber\\ 
X_{ij}\sim {\mathcal{N}}(0,1/n) \ \ , \ \ 
\ve\sim {\mathcal{N}}(0,\sigma^2 \mI)) 
\label{eq:gen_gauss}
\end{eqnarray}

We compare MINIMAX to the baselines LR, LPR, STE and SDR. All methods return a vector $\vw\in \{\pm 1 \}^n$ and accuracy is reported as the fraction of correct values compared to $\vw^{\ast}$, namely $\frac{1}{2m}\|\vw^{\ast}-\vw\|_1$.


In Figure \ref{small_scale_blr}, we report results of a small scale experiment. We use log scale in the error axis to emphasize the low noise variance regime. In this setting, SDR is known to be near optimal \cite{QCQP_SDR} and MAXMIN succeeds in matching the SDR performance. STE does not perform well in this case.

Next in Figure \ref{large_scale_blr}, we present results of a larger scale experiment with $p=2000$ and $n=4000$. In this regime, optimizing SDR is not feasible as its complexity depends on $n^2$ variables. MAXMIN continues to dominate the other methods in this large scale setting as well. Interestingly, STE performs well when the noise is large but is inferior to simple LR in low noise conditions.

\subsection{Synthetic Robust Binary Regression}
In \secref{sec:huberloss} we considered the case of regression with outliers, and used the Huber's loss. More precisely, to minimize the number of hyperparameters, we experiment with an L$_1$ loss that corresponds to Huber's loss with $
\delta\rightarrow 0$. We consider data generated as in (\ref{eq:gen_gauss}) but replace the noise with Laplace distribution:
$e_i\sim {\mathcal{L}aplace}(0,\sigma), \forall i$.
We compare the error probabilities of LR, LPR, SDR, STE and MAXIMIN, with both L$_1$ and L$_2$ losses. 
The results are presented in Figure \ref{large_scale_rblr}. When the noise variance is small, L$_1$ MAXMIN provides near perfect recovery as predicted by our theory. With larger noise, MAXMIN continues to provide a steady performance advantage over the baselines.

\subsection{Boston house prices dataset}
Next, we consider more realistic experiments of binary constraint optimization on real world datasets, here we'll consider a regression problem.
We begin with the Boston house prices database benchmark where the goal is to predict a house price using various features such as crime rate, number of rooms, etc. Table 
\ref{boston_data} shows the results of our model compared to all methods discussed above. We used Normalized Root Mean Squared Error (NRMSE) as the evaluation metric:
\begin{equation}
    \text{NRMSE}\left(\vu,\hat{\vu} \right) = \frac{\|\vu-\hat{\vu}\|_2}{\|\vu\|_2},
\end{equation}
The first column reports the performance over the test set where the original dataset was split to 70\% train and 30\% test. As expected SDR and MAXIMIN L$_2$ outperform the baseslines. Interestingly, MAXIMIN L$_2$ is even slightly better\footnote{Our SDP binarization was implemented by rounding the top eigenvector. The performance of SDR can be slightly improved using randomized binarization \cite{QCQP_SDR}.}.
In the second column, we present results using a corrupted dataset. We artificially add 25\% outliers to the train set and repeat the experiemnts. It is easy to see the degradation in performance due to the outliers, as well as the robustness on MAXIMIN L$_1$. 

\subsection{Nonlinear Classification on MNIST} \label{sec:exp_deep}
Finally, we present experiments in quantized classification  using the popular MNIST dataset. 
We define a prediction function $f(\vx;\vw)$ on input $\vx$ parameterized by $\vw$, together with a cross entropy loss denoted by $\ell()$ and consider:
\begin{equation}
\left\{\begin{array}{ll}
    \min_\vw & \sum_j \ell(y_j,f(\vx_j;\vw)) \\
    {\rm{s.t.}} & w_i^2=1,\forall i
\end{array}
\right.
\end{equation}
We explore three types of networks: linear, one hidden layer (128 hidden neurons) and a convolution neural network (CNN) with architecture $32C3 - MP2 - 64C3 - MP2 - 128C3 - 3\times128FC$.\footnote{Architecture description: MPx: Max Pooling with stride x, xCy: Convolution of x channels with kernel size y, xFC: Fully Connected with x outputs.} We used ReLU as the activation function, and added batch normalization between each two layers.


Table \ref{mnist_test_acc} shows the accuracy results of the unconstrained real-valued networks, and binary-valued networks learned using STE, MAXIMIN and BOP. It can be seen that MAXIMIN outperforms both STE and BOP on the linear network, and the two methods are comparable for the other architectures. Given our theory, the results of the linear model are expected, but this experiment illustrates the  potential of MAXIMIN beyond the convex regime.





Figure \ref{mnist_weights_histogram} shows the histogram of the CNN weights during the MAXIMIN optimization. Each histogram corresponds to a different layer, and each slice (z axis) corresponds to an iteration. It can clearly be seen that MAXIMIN weights converge to binary weights quickly and uniformly over all layers. Thus we can conclude that MAXIMIN indeed leads to learning binarized weights with high accuracy in this case.

For completeness, we note that additional experiments (not shown) using deeper architectures led to inferior MAXIMIN performance. The combination of highly non-convex optimizations with binary constraints is very challenging and outside the scope of this paper.

\begin{figure*}[t]
\centering
\includegraphics[width=1.0\textwidth]{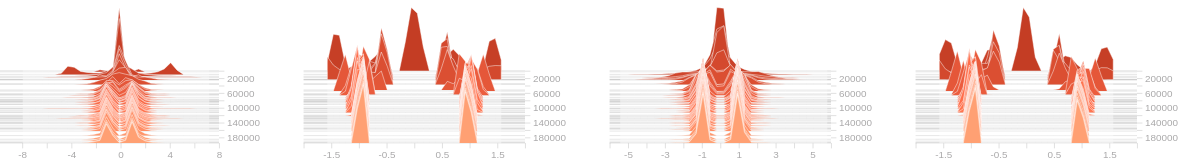}
\includegraphics[width=1.0\textwidth]{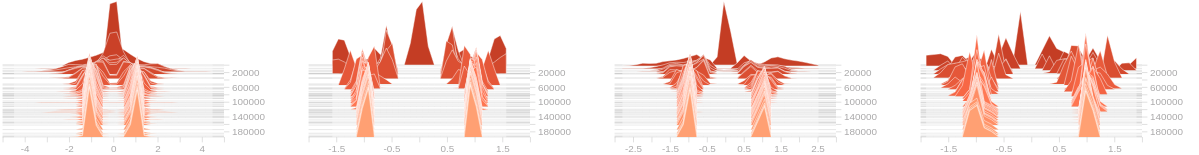}
\caption{A CNN model for MNIST, learned with MAXIMIN. Shown are the histograms of $w$ values for different layers (one figure per layer) . The rows in each figure correspond to different optimization steps.}\label{mnist_weights_histogram}
\end{figure*}



\begin{table}[h]
    \centering
    \begin{tabular}{|l|c|c|}
    \hline
         & Clean Dataset & 25\% outliers \\
         \hline
        LR L1 &      --             & 5.6 $\pm$ 0.9  \\
        LR L2 &      3.7 $\pm$ 0.7  & --  \\
        LPR L1 &     --             & 5.6 $\pm$ 0.9  \\
        LPR L2 &     3.7 $\pm$ 0.7  & --  \\
        STE L1 &     --             & 11.3 $\pm$ 1.6  \\
        STE L2 &     7.1 $\pm$ 1.9  & --  \\
        SDR &        2.5 $\pm$ 0.4  & 3.5 $\pm$ 0.5  \\
        MAXIMIN L1 & 2.6 $\pm$ 0.3  & \textbf{3.2 $\pm$ 0.4}  \\
        MAXIMIN L2 & \textbf{2.4 $\pm$ 0.2}  & 3.5 $\pm$ 0.7  \\
        \hline
    \end{tabular}
    \caption{Boston Housing Test NRMSE}
    \label{boston_data}
\end{table}

\begin{table}[h]
    \centering
    \begin{tabular}{|l|c|c|c|}
    \hline
         &  Linear & One Layer & CNN\\
         \hline
        Real valued & 0.92 & 0.98 & 0.99 \\
        \hline
        STE & 0.87 & 0.96 & 0.99 \\
        BOP & 0.87 & 0.97 & 0.99 \\
        MAXIMIN & 0.9 & 0.97 & 0.99 \\
        \hline
    \end{tabular}
    \caption{Test Accuracy on MNIST}
    \label{mnist_test_acc}
\end{table}

\section{Discussion}
Learning quantized deep-learning models is essentially a discrete optimization problem. However, due to the highly non-linear nature of deep-learning models it is not clear how to extend common discrete optimization methods to this setting. Here we showed that one of the most powerful discrete optimization methods, namely Lagrangian or Semidefinite Relaxations, can in fact be cast in a way that allows scalable optimization and extends to the non-linear case. 

Our theoretical results assume a certain restriction on the noise level in the observations. However, in practice we observe that the method matches the performance of SDR beyond these noise levels. It would be interesting to further strengthen the results accordingly. Furthermore, it would be interesting to extend the theory to other loss functions such as logistic regression and to more complex architectures such as one hidden layer networks. Another important extension is to consider MAXIMIN in regression with more levels of quantization. 

Our theory and experiments show promising performance to MAXIMIN in convex (or near-convex) minimization with binary variables. On the negative side, we did not succeed to use it in deep non-convex architectures. Without the binary constraints, such optimizations are theoretically challenging but are easily solved in practice these days. Our trials with deep MAXIMINs led to unstable behavior. Future work should address such settings and understand whether this is a fundamental barrier, or a technical challenge requiring better parameter tuning and algorithms.

\appendix
Here we will prove the three main results \ref{general_convex_thm}, \ref{blr_convex_thm}, \ref{robust_convex_thm} stated in Section \ref{sec:main}.
\section*{Proof for Theorem \ref{general_convex_thm}}
We begin by recalling the necessary and sufficient conditions for a local maximin point.
Note that these conditions are more specific than those in \cite{JinMinimax2019} and are simplified due to the fact that the lagrangian is linear in the dual variables.

\begin{proposition}[Lagrangian local maximin - sufficient Conditions \cite{JinMinimax2019}] \label{maximin_suff}
Assume that $\mathcal{L}$ is twice differential. Any stationary
point $\left(\vw,\vz \right)$ satisfying
\begin{eqnarray}
\nabla_{\vw}^{2}\mathcal{L}\left(\vw,\vz\right) \succ0, 
&\nabla_{\vw \vz}^{2}\mathcal{L}\left(\vw,\vz\right) \text{is invertible}
\end{eqnarray}
is a local maximin.
\end{proposition}

\begin{proposition}[Lagrangian Local maximin - Necessary Conditions \cite{JinMinimax2019}] \label{maximin_ness}
Assuming $\mathcal{L}$ is twice-differentiable, any local maximin $ \left( \vw,\vz \right)$ satisfies
\begin{eqnarray}
\nabla \mathcal{L}\left(\vw,\vz\right)	&=0 \quad
\nabla_{w}^{2}\mathcal{L}\left(\vw,\vz\right)	&\succeq0.
\end{eqnarray}
\end{proposition}

In the case of binary optimizations, these conditions lead to the following results.

\begin{proposition}[Binary sufficient conditions]
Let $\vw^\ast$ be an optimal solution to \eqref{BG}, under the assumption: 
\begin{eqnarray}
&&\|\nabla f \left( \vw^\ast \right) \|_\infty < \lambda_{\min} \left(   \nabla^2 f \left( \vw^\ast \right) \right)
\end{eqnarray}
The pair:
\begin{eqnarray}\label{global_true}
\left( \vw^*,\frac{1}{2} \diag{ \vw^{\ast}} \nabla f \left( \vw^\ast \right) \right)
\end{eqnarray}
is a local maximin point of \eqref{BG_lagrange}.
\end{proposition}
\begin{proof}
First, $\vw^{\ast}$ is optimal and feasible, i.e.
\begin{eqnarray}
w^{\ast 2}_i &=1, \forall i, \quad
\end{eqnarray}
set:
\begin{flalign}
\vz^{\ast} &= \frac{1}{2} \diag{ \vw^{\ast}} \nabla f \left( \vw^\ast \right)
\end{flalign}
which means
\begin{flalign}
\nabla^2_{wz} \mathcal{L} \left( \vw^{\ast},\vz^{\ast} \right) = -2\diag{\vw^{\ast}}
\end{flalign}
is invertable, and also:
\begin{flalign}
\nabla^2_w \mathcal{L} \left( \vw^{\ast},\vz^{\ast} \right) & =
\nabla^2 f \left( \vw^\ast \right) -
2\diag{\vz^{\ast}} \nonumber \\
&= \nabla^2 f \left( \vw^\ast \right) -
\diag{ \vw^{\ast}} \diag{\nabla f \left( \vw^\ast \right)} \nonumber \\
& \succeq
\nabla^2 f \left( \vw^\ast \right) -
 \|\nabla f \left( \vw^\ast \right)\|_{\infty}\mI \nonumber \\
& \succ
\nabla^2 f \left( \vw^\ast \right) -
\lambda_{min} \left( \nabla^2 f \left( \vw^\ast \right) \right) \mI \nonumber \\
& \succeq \vzero
\end{flalign}
Due to Proposition \ref{maximin_suff}, $\left( \vw^{\ast}, \vz^{\ast}\right)$ is a maximin point.
\end{proof}


\begin{proposition}[Binary necessary conditions]
If $\left(\vw,\vz\right)$ is a local maximin point of (\ref{BG_lagrange}) and  $f$ is a sub-quadratic function then $\vw$ is an optimal solution to (\ref{BG}).
\end{proposition}

\begin{proof}
Let $\left(\vw,\vz \right)$ be a local maximin point of (\ref{BG_lagrange}) then the following holds:
\begin{eqnarray}
\nabla_{\vz}\mathcal{L}\left(\vw,\vz\right)=0 \label{gen_grad_dual}
\Leftrightarrow	&
\vw_i^2=1,\forall i \\
\nabla_{\vw}\mathcal{L}\left(\vw,\vz\right)=0 \label{gen_grad_primal}
\Leftrightarrow	&
\nabla f \left( \vw\right) -2\diag{\vz}\vw=0 \\
\nabla_{\vw}^{2}\mathcal{L}\left(\vw,\vz\right) \succeq0 \label{gen_hess_primal}
\Leftrightarrow	&
\nabla^2 f \left( \vw\right) -2\diag{\vz}\succeq0
\end{eqnarray}
Assume towards contradiction that $\vw$ in not the global optimum of (\ref{BG}), i.e. $f \left( \vw\right) > f \left( \vw^{\ast}\right)$,
using the above conditions we obtain:
\begin{flalign}
0 &=\left(\vw-\vw^{\ast}\right)^{T}\nabla f\left(\vw\right)-2\left(\vw-\vw^{\ast}\right)^{T}\diag{\vz}\vw \nonumber\\
  &=\left(\vw-\vw^{\ast}\right)^{T}\nabla f\left(\vw\right)-\left(\vw-\vw^{\ast}\right)^{T}\diag{\vz}\left(\vw-\vw^{\ast}\right) \nonumber\\
  &\geq\left(\vw-\vw^{\ast}\right)^{T}\nabla f\left(\vw\right)-\frac{1}{2}\left(\vw-\vw^{\ast}\right)^{T}\nabla^{2}f\left(\vw\right)\left(\vw-\vw^{\ast}\right) \nonumber
\end{flalign}
where the first equality is due to the fact that $\vw$ is stationary point of $f$,
the second equality is because $\vw, \vw^{\ast} \in \{ \pm 1 \}^n$, and the third inequality is due to (\ref{gen_hess_primal}).
Finally, recall that $f$ is sub-quadratic and thus:
\begin{flalign}
  0&< \left(\vw-\vw^{\ast}\right)^{T}\nabla f\left(\vw\right)-\frac{1}{2}\left(\vw-\vw^{\ast}\right)^{T}\nabla^{2}f\left(\vw\right)\left(\vw-\vw^{\ast}\right) \nonumber
\end{flalign}
which is a contradiction.
\end{proof}



\section*{Proof for Theorem \ref{blr_convex_thm}}
To establish this theorem we just need to show that the conditions of Proposition \ref{general_convex_thm} holds with the linear regression objective:
\begin{eqnarray}
 f \left(\vw\right) = \|\vX\vw-\vy\|_2^2,\quad \vy=\vX\vw^*+\ve
\end{eqnarray}
Thus:
\begin{eqnarray}
 &&f \left(\vw \right) = \|\vX \left( \vw-\vw^{\ast} \right) - \ve\|_2^2 \nonumber \\
&&\nabla f \left(\vw \right) =2\mX^T(\mX(\vw-\vw^*)+\ve)\nonumber\\
&&\nabla^2 f \left( \vw \right) =2\mX^T\mX
\end{eqnarray}
The condition follows immediately:
\begin{eqnarray}
&&\|\nabla f \left( \vw^\ast \right) \|_\infty < \lambda_{\min} \left(   \nabla^2 f \left( \vw^\ast \right) \right) \nonumber \\ 
&& \Leftrightarrow \nonumber \\ 
&&\|\vX^T\ve\|_\infty < \lambda_{\min} \left( \vX^T\vX \right)
\end{eqnarray}
Finally, a quadratic function is obviously sub-quadratic, if $ f\left(\vw_1\right) < f\left(\vw_2\right)$ then:
\begin{multline}
0> f\left(\vw_1\right) - f\left(\vw_2\right) = \left(\vw_1-\vw_2\right)^{T} \nabla f\left(\vw_2\right) \\
 +\frac{1}{2}\left(\vw_1-\vw_2\right)^{T}\nabla^{2}f\left(\vw_2\right)\left(\vw_1-\vw_2\right) \nonumber
\end{multline}
where the equality is due to the fact that second order Taylor expansion is exact for quadratic objective.

\section*{Proof for Theorem \ref{robust_convex_thm}}
We need to show that the conditions of Proposition \ref{general_convex_thm} hold with the robust regression objective:
\begin{eqnarray}
f \left( \vw\right) = \sum_i\ell \left( \vx_i^T\vw-y_i \right), \quad \vy=\vX\vw^*+\ve
\end{eqnarray}
where:
\begin{equation}
\ell \left( z \right) = 
     \begin{cases}
       \frac{z^2}{2} &\quad |z|\leq \delta\\ 
       |z|\delta - \frac{\delta^2}{2} &\quad \text{else} 
     \end{cases}  \nonumber
\end{equation}
Thus:

\begin{multline}
 f \left(\vw \right) = \sum_{i\in I} \frac{1}{2}\left( \vx_i^T \left( \vw - \vw^{\ast}\right) - e_i \right)^2  \\
 + \sum_{i\notin I} \left| \vx_i^T \left( \vw - \vw^{\ast}\right) - e_i \right|+c
\end{multline}
\begin{multline}
\nabla f \left( \vw \right) = \sum_{i\in I} \left( \vx_i^T \left( \vw - \vw^{\ast}\right) - e_i \right)\vx_i \\
+ \sum_{i\notin I} \sign\left( \vx_i^T \left( \vw - \vw^{\ast}\right) - e_i \right)\vx_i
\end{multline}
\begin{flalign}
\quad \nabla^2 f \left( \vw \right) =\sum_{ i\in I} \vx_i \vx_i^T&&
\end{flalign}
where $I$ are the set of indices for which $|\vx_i^T \left( \vw - \vw^{\ast}\right) - e_i| \leq \delta$. 
Again, the general condition holds immediately:
\begin{eqnarray}
&&\|\nabla f \left( \vw^\ast \right) \|_\infty < \lambda_{\min} \left(   \nabla^2 f \left( \vw^\ast \right) \right) \nonumber \\ 
&& \Leftrightarrow \nonumber \\ 
&&\|\vX^T\vc\|_\infty < \lambda_{\min} \left( \sum_i d_i \vx_i\vx_i^T \right)
\end{eqnarray}


What remains is to prove that $f$ is sub-quadaratic. For this purpose, we define a strongly sub-quadratic function
\begin{definition}
[Strongly Sub-quadratic function]
A function $f$ is called strongly sub-quadratic if $\forall \vw_1, \vw_2$ s.t. $f\left(\vw_1\right) < f\left(\vw_2\right)$
\begin{multline}
f\left(\vw_1\right) - f\left(\vw_2\right) \geq \left(\vw_1-\vw_2\right)^{T} \nabla f\left(\vw_2\right) \\
 +\frac{1}{2}\left(\vw_1-\vw_2\right)^{T}\nabla^{2}f\left(\vw_2\right)\left(\vw_1-\vw_2\right)
\end{multline}
\end{definition}
This condition is stronger than the sub-quadratic condition, i.e. if $f$ is strongly sub-quadratic function then $f$ is also a sub-quadratic function.
 

\begin{lemma}[Huber's loss is strongly\label{lemma_huber} sub-quadratic]
Huber's loss:
\begin{equation}
\ell \left( z \right) = 
     \begin{cases}
       \frac{z^2}{2} &\quad |z|\leq \delta\\ 
       |z|\delta - \frac{\delta^2}{2} &\quad \text{else} 
     \end{cases}  \nonumber
\end{equation}
is strongly sub-quadratic
\end{lemma}
\begin{proof}
First, $\ell \left(z_1\right) \leq \ell \left(z_2\right)$, so if $|z_2| \leq \delta$ then also $|z_1| \leq \delta$, and since both points are on the quadratic regime of the loss, we trivially get:
\begin{eqnarray}
\ell \left(z_1\right) -  \ell \left(z_2\right) \geq 
\ell ' \left(z_2\right) \left(z_1-z_2\right) 
 +\frac{1}{2}\ell '' \left(z_2\right)\left(z_1-z_2\right)^2
\end{eqnarray}
If $|z_2| > \delta$ then it is on the linear regime of the loss, thus $ \ell '' \left(z_2\right) = 0$ and from convexity of huber loss:
\begin{eqnarray}
\ell \left(z_1\right) -  \ell \left(z_2\right) &\geq 
\ell ' \left(z_2\right) \left(z_1-z_2\right).
\end{eqnarray}
\end{proof}

\begin{lemma}[Robust linear regression is sub-quadratic]
Let $\ve$ be the noise vector such that for every $i$: $|e_i| \leq \delta $ or $|e_i| \geq \delta+2\| x_i\|$
then linear regression with Huber's loss is sub-quadratic.
\end{lemma}
\begin{proof}
Define the set of all indices in which $\vw^{\ast}$ is on the linear regime of the huber loss and $\vw$ is on the quadratic regime regime:
\begin{eqnarray}
B = \{ i : \ell \left( \vx_i^T \vw - y_i \right) \leq \frac{1}{2}\delta^2 \leq \ell \left( \vx_i^T \vw^{\ast} - y_i \right)\}
\end{eqnarray}
then:
\begin{flalign}
\left(\vw^{\ast}- \vw\right)^{T}& \nabla f\left(\vw\right)
 +\frac{1}{2}\left(\vw^{\ast}-\vw\right)^{T}\nabla^{2}f\left(\vw\right)\left(\vw^{\ast}-\vw\right) \nonumber \\
=& \sum_i \ell ' \left( \vx_i^T \vw- y_i \right)\vx_i^T \left(\vw^{\ast}-\vw\right) \nonumber \\
\label{robust_subq2}& +\frac{1}{2} \sum_i \ell '' \left( \vx_i^T \vw- y_i \right) \left( \vx_i^T\left(\vw^{\ast}-\vw\right) \right)^2  \\
\leq & \sum_i \ell \left( \vx_i^T \vw^{\ast} - y_i \right) -  \ell \left( \vx_i^T \vw- y_i \right) \nonumber \\
\label{robust_subq3}& +\frac{1}{2} \sum_{i\in B} \ell '' \left( \vx_i^T \vw- y_i \right) \left( \vx_i^T\left(\vw^{\ast}-\vw\right) \right)^2  \\
\label{robust_subq4} < & 0
\end{flalign}

To show the inequality in (\ref{robust_subq3}), we consider four cases:
\begin{itemize}
    \item $\vw^{\ast}, \vw$ are both on the linear regime $\ell \left( \vx_i^T \vw - y_i \right),  \ell \left( \vx_i^T \vw^{\ast} - y_i \right) \geq \frac{\delta^2}{2}$: the linear regime is strongly sub-quadratic.
    \item $\vw^{\ast}, \vw$ are both on the quadratic regime $\ell \left( \vx_i^T \vw - y_i \right),  \ell \left( \vx_i^T \vw^{\ast} - y_i \right) \leq \frac{\delta^2}{2}$: the quadratic regime is strongly sub-quadratic.
    \item  $\vw^{\ast}$ has a lower objective $\ell \left( \vx_i^T \vw - y_i \right) \geq \ell \left( \vx_i^T \vw^{\ast} - y_i \right)$: due to Lemma \ref{lemma_huber} $\ell$ is strongly sub-quadratic.
    \item $i \in B$: Due to convexity of Huber's loss:
    \begin{multline}
    \ell \left( \vx_i^T \vw^{\ast} - y_i \right) -  \ell \left( \vx_i^T \vw- y_i \right) \\
    \geq  \ell ' \left( \vx_i^T \vw- y_i \right)\vx_i^T \left(\vw^{\ast}-\vw\right)
    \end{multline}
\end{itemize}

Finally, the inequality in (\ref{robust_subq4}) is due to the fact that $B = \emptyset$ under the lemma noise constraints, and by definition $f\left( \vw^{\ast} \right) < f\left( \vw \right)$.
\end{proof}

\section*{Proof for Theorem \ref{huber_diversity}}


First, recall the settings, the objective is:
\begin{equation}
f \left( \vw\right) = \sum_i\ell \left( \vx_i^T\vw-y_i \right)\\
\end{equation}
where $\ell$ is Huber's loss, assume the data is generated as follows:
\begin{itemize}
    \item $\vw^*\in\{\pm 1\}^n$
    \item $\vy=\vX\vw^*+\ve$
    \item $\vX\in {\mathbb{R}}^{m\times n}$ such that $X_{i,j}\sim {\mathcal{N}}(0,1)$ for every $i \in [n], j \in [m]$ independently.
    \item $\ve^*\in{\mathbb{R}}^{m}$ 
    such that on $pm$ indices\footnote{A finer analysis can also handle a random number of outliers but is omitted for simplicity.} $e_i$ is arbitrary large and $e_i=0$ on the rest, for a constant $0\leq p<\frac{1}{2}$.
\end{itemize}

Next, we bound the probability of the event that:
\begin{eqnarray}
&& \max_m |{\mathcal{N}}(0,k)| \leq \lambda_{\min}(\sum^{m-k}_i\vx_i\vx_i^T)
\end{eqnarray}
We define \begin{eqnarray}
\|v\|_\infty\leq \lambda_{\min}(\mA) \nonumber
\end{eqnarray}
where $v$ is a vector with $n$ independent ${\mathcal{N}}(0,k)$ variables, and $\mA$ is an $n\times n$ Wishart matrix with $m-k=r$ degrees of freedom.

Using Corollary 5.35 in \cite{Vershynin2010Wishart} and setting $t=\frac{1}{2}\sqrt{r}$, we get the bound:
\begin{eqnarray}
P \left(\lambda_{\min}(\mA)\leq \left( \frac{3}{2} \sqrt{r} + \sqrt{n}\right)^2 \right)
\leq 2e^{-\frac{r}{8}} \label{div_1}
\end{eqnarray}
Using Gaussian bound from Section 7.1 in \cite{feller-vol-1} and setting $x=\frac{1}{\sqrt{k}}\left( \frac{3}{2} \sqrt{r} + \sqrt{n}\right)^2$, we get the bound:
\begin{flalign}
P&\left(|v_i|\geq \left( \frac{3}{2} \sqrt{r} + \sqrt{n}\right)^2\right) &\nonumber\\
&\qquad \qquad \leq \sqrt{\frac{2k}{\pi}}\frac{1}{\left( \frac{3}{2} \sqrt{r} + \sqrt{n}\right)^2} 2ne^{- \frac{1}{2k}\left( \frac{3}{2} \sqrt{r} + \sqrt{n}\right)^4} \nonumber \\
&\qquad \qquad \leq  \sqrt{\frac{2k}{\pi}}e^{- \frac{r^2}{2k}} \nonumber \\
&\qquad \qquad \leq  \sqrt{\frac{m}{\pi}}e^{- \frac{m}{4}}
\end{flalign}
where in the last two inequalities we used the theorem assumptions $r \geq \frac{1}{2}m \geq k \geq 0$. 
Their union bound  yields:
\begin{flalign}
P & \left(\exists i,  |v_i|\geq \left( \frac{3}{2} \sqrt{r} + \sqrt{n}\right)^2\right) & \nonumber \\
& \qquad \qquad \qquad \leq \sum_i P\left(|v_i|\geq \left( \frac{3}{2} \sqrt{r} + \sqrt{n}\right)^2\right) \nonumber \\
&\qquad \qquad \qquad \leq n \cdot \sqrt{\frac{m}{\pi}}e^{- \frac{m}{4}}
 \label{div_2}
\end{flalign}
Finally, combining the two bounds (\ref{div_1}) and (\ref{div_2}), results in the condition:
\begin{flalign}
P \left( \|v\|_\infty\leq \lambda_{\min}(\mA) \right) 
& \geq 1 - n \cdot \sqrt{\frac{m}{\pi}}e^{- \frac{m}{4}} - 2e^{-\frac{m}{16}} \nonumber \\
& \geq 1 - 5n \cdot \sqrt{\frac{m}{\pi}}e^{- \frac{m}{16}}
\end{flalign}
Where in the last inequality we used again the assumption $m > n \geq 1$.

\clearpage
\bibliography{ref.bib}

\ifCLASSOPTIONcaptionsoff
  \newpage
\fi

\end{document}